\documentclass{article}
\usepackage[utf8]{inputenc}
\usepackage[margin=1in]{geometry}
\usepackage[numbers,sort]{natbib}
\usepackage{csvsimple,booktabs,adjustbox}
\usepackage{float}
\usepackage{aux/macros}
\usepackage{pifont}
\usepackage{enumitem}
\newcommand\contributionNote[1]{%
  \begingroup
  \renewcommand\thefootnote{}\footnote{\kern-5pt \textcolor{white}{\rule{5pt}{2ex}}#1}%
  \addtocounter{footnote}{-1}%
  \endgroup
}

\def\hat{\widehat}
\def\tilde{\widetilde}

\begin{document}
\begin{center}
    \vspace*{0.5cm}

\LARGE{Mechanism of feature learning in convolutional neural networks} 
 
\vspace*{0.5cm}

\large{Daniel Beaglehole$^{*, 2}$ \\ Adityanarayanan Radhakrishnan$^{*, 3, 4}$ \\ Parthe Pandit$^{1}$ \\ \hspace{.1mm} Mikhail Belkin$^{1, 2}$}

\vspace*{0.5cm}

\normalsize{$^{1}$Hal\i c\i o\u glu Data Science Institute, UC San Diego.} \\
\normalsize{$^{2}$Computer Science and Engineering, UC San Diego.} \\
\normalsize{$^{3}$Massachusetts Institute of Technology.} \\
\normalsize{$^{4}$Broad Institute of MIT and Harvard.} \\
\normalsize{$^{*}$Equal contribution.}
\vspace*{0.5cm}
\end{center}

\setcounter{footnote}{3}
\begin{abstract} 
Understanding the mechanism of how convolutional neural networks learn features from image data is a fundamental problem in machine learning and computer vision.  In this work, we identify such a mechanism.  We posit the Convolutional Neural Feature Ansatz, which states that covariances of filters in any convolutional layer are proportional to the average gradient outer product (AGOP) taken with respect to patches of the input to that layer.  We present extensive empirical evidence for our ansatz, including identifying high correlation between covariances of filters and patch-based AGOPs for convolutional layers in standard neural architectures, such as AlexNet, VGG, and ResNets pre-trained on ImageNet.  We also provide supporting theoretical evidence. We then demonstrate the generality of our result by using the patch-based AGOP to enable deep feature learning in convolutional kernel machines.  We refer to the resulting algorithm as (Deep) ConvRFM and show that our algorithm recovers similar features to deep convolutional networks including the notable emergence of edge detectors.  Moreover, we find that Deep ConvRFM overcomes previously identified limitations of convolutional kernels, such as their inability to adapt to local signals in images and, as a result, leads to sizable performance improvement over fixed convolutional kernels.  
\end{abstract}

\section{Introduction}

Neural networks have achieved impressive empirical results across various tasks in natural language processing~\cite{GPT3}, computer vision~\cite{DALLE}, and biology~\cite{Alphafold}.  Yet, our understanding of the mechanisms driving the successes of these models is still emerging.  One such mechanism of central importance is that of \textit{neural feature learning}, which is the ability of networks to automatically learn relevant input transformations from data~\cite{FeatureLearningEmergenceShi, FeatureLearningEmpiricalEvidence, YangFeatureLearning, radhakrishnan2022feature}.   

An important line of work~\cite{FeatureLearningEmergenceShi, YangFeatureLearning, jacot2022feature, PreetumLimitations, DamianLowRank, WangOneStep, IoannisLowRank, parkinson2023linear} has demonstrated how feature learning in fully connected neural networks provides an advantage over classical, non-feature-learning models such as kernel machines.  Recently, the work~\cite{radhakrishnan2022feature} identified a connection between a mathematical operator, known as average gradient outer product (AGOP)~\cite{trivedi2020expected, EGOPSamory, RecursiveMultiIndex, HardleStokerGradientAveraging},  and feature learning in fully connected networks.  This work subsequently demonstrated that the AGOP could be used to enable similar feature learning in kernel machines operating on tabular data.  In contrast to the case for fully connected networks, there are few prior works~\cite{karp2021local, DeepFL-Theory} analyzing feature learning in convolutional networks, which have been transformative in computer vision~\cite{ResNet, DALLE}.  The work~\cite{karp2021local} demonstrates an advantage of feature learning in convolutional networks by showing that these models are able to threshold noise and identify signal in image data unlike convolutional kernel methods including Convolutional Neural Tangent Kernels~\cite{CNTKArora}.  The work~\cite{DeepFL-Theory} analyzes how deep convolutional networks can correct features in early layers by simultaneous training of all layers.  While these prior works identify advantages of feature learning in convolutional networks, they do not identify a general operator that captures such feature learning.  The connection between AGOP and feature learning in fully connected neural networks~\cite{radhakrishnan2022feature} suggests that a similar connection should exist for feature learning in convolutional networks.  Moreover, such a mechanism could be used to learn analogous features with any machine learning model such as convolutional kernel machines.   

In this work, we establish a connection between convolutional neural feature learning and the AGOP, which we posit as the Convolutional Neural Feature Ansatz (CNFA).  Unlike the fully connected case from~\cite{radhakrishnan2022feature} where feature learning is characterized by AGOP with respect to network inputs, we demonstrate that convolutional feature learning is characterized by AGOP with respect to patches of network inputs.  We present empirical evidence for the CNFA by demonstrating high average Pearson correlation (in most cases $> .9$) between AGOP on patches and the covariance of filters across all layers of pre-trained convolutional networks on ImageNet~\cite{ImageNet} and across all layers of SimpleNet~\cite{SimpleNet} trained on several standard image classification datasets.  We additionally prove that the CNFA holds for one step of gradient descent for deep convolutional networks.  To demonstrate the generality of our identified convolutional feature learning mechanism, we leverage the AGOP on patches to enable feature learning in convolutional kernel machines.  We refer to the resulting algorithm as ConvRFM.  We demonstrate that ConvRFM captures features similar to those learned by the first layer of convolutional networks.  In particular, on various image classification benchmark datasets such as SVHN~\cite{SVHN} and CIFAR10~\cite{CIFAR10}, we observe that ConvRFM recovers features corresponding to edge detectors.  We further enable deep feature learning with convolutional kernels by developing a layerwise training scheme with ConvRFM, which we refer to as Deep ConvRFM.  We demonstrate that Deep ConvRFM learns features similar to those learned by deep convolutional neural networks. Furthermore, we show that Deep ConvRFM overcomes limitations of convolutional kernels identified in~\cite{karp2021local} and exhibits {\it local feature adaptivity}. Lastly, we demonstrate that Deep ConvRFM provides improvement over CNTK and ConvRFM on several standard image classification datasets, indicating a benefit to deep feature learning.  Our results advance understanding of how convolutional networks automatically learn features from data and provide a path toward integrating convolutional feature learning into general machine learning models.

\section{Convolutional Neural Feature Ansatz (CNFA)}

Let $f: \mathbb{R}^{c \times P \times Q} \to \mathbb{R}$ denote a convolutional neural network (CNN) operating on $P \times Q$ resolution images with $c$ color channels.  The $\ell^{th}$ convolutional layer of a CNN involves applying a function $h_{\ell}: \mathbb{R}^{c_{\ell-1} \times P_{\ell-1} \times Q_{\ell-1}} \to \mathbb{R}^{c_{\ell} \times P_{\ell} \times Q_{\ell}}$ defined recursively as $h_{\ell}(x) = \phi(\tilde{W}_{\ell} * h_{\ell-1}(x))$ with $h_{1} = x$, $\tilde{W}_{\ell} \in \mathbb{R}^{c_{\ell} \times c_{\ell-1} \times q \times q}$ denoting $c_{\ell}$ filters of size $c_{\ell-1} \times q \times q$, $*$ denoting the convolution operation, and $\phi$ denoting an elementwise activation function.  To understand how features emerge in convolutional networks, we abstract a convolutional network to a function of the form 
\begin{align}
\label{eq: Convolutional Abstraction}
f(x) = g(W_1 x[1, 1], \ldots, W_1 x[i,j],\ldots,W_1 x[P, Q]),\quad i\in[P], j\in[Q]~;     
\end{align}
where $W_1 \in \mathbb{R}^{c_{1} \times cq^2}$ is a matrix of $c_1$ stacked filters of size $cq^2$ and $x[i, j] \in \mathbb{R}^{cq^2}$ denotes the patch of $x$ centered at coordinate $(i, j)$.  This abstraction is helpful since it allows us to consider feature learning in convolutional networks with arbitrary architecture (e.g., pooling layers, batch normalization, etc.) after any given convolutional layer. Up to rotation and reflection by the left singular vectors, the feature extraction properties of $W_1$ are determined by the singular values and right singular vectors of $W_1$.  These singular values and vectors can be recovered from the matrix $W_1^T W_1$, which is the empirical (uncentered) covariance of filters in the first layer.  This argument extends to analyze features selected at layer $\ell$ of a CNN by considering a function of the form $f(x) = g_{\ell}(W_\ell h_{\ell-1}(x)[1, 1], \ldots,  W_\ell h_{\ell-1}(x)[P_{\ell-1}, Q_{\ell-1}])$.  We refer to the matrix $W_{\ell}^T W_{\ell}$ as a \textit{Convolutional Neural Feature Matrix} (CNFM) and note that this matrix is proportional to the (uncentered) empirical covariance matrix of filters in layer $\ell$.  

We use the form of convolutional networks presented in Eq.~\eqref{eq: Convolutional Abstraction} to state our Convolutional Neural Feature Ansatz (CNFA).  Let $G_{\ell}(x) := g_{\ell}(W_\ell h_{\ell-1}(x)[1, 1], \ldots,  W_\ell h_{\ell-1}(x)[P_{\ell-1}, Q_{\ell-1}])$.  Then, after training $f$ for at least one epoch of (stochastic) gradient descent on standard loss functions:
\begin{align}
\label{eq: CNFA}
    W_{\ell}\tran W_{\ell} \propto \sum_{p=1}^{n} \sum_{(i, j) \in S} \nabla_{h_{\ell-1}(x)[i,j]} G_{\ell}(x) \left(\nabla_{h_{\ell-1}(x)[i,j]} G_{\ell}(x)\right) \tran; 
\end{align}
where $S = \{(i,j)\}_{i \in [P_{\ell-1}], j \in [Q_{\ell-1}]}$ denotes the set of indices of patches utilized in the convolution operation in layer $\ell$.  The CNFA (Eq.~\ref{eq: CNFA}) mathematically implies that the convolutional neural feature matrices are proportional to the average gradient outer product (AGOP) with respect to the patches of the input to layer $\ell$.  The CNFA implies that the structure of covariance matrices of filters in convolutional networks, an object studied in prior work~\cite{TrockmanCovariance}, corresponds to AGOP over patches.  Intuitively, the CNFA implies that convolutional features are constructed by identifying and amplifying those pixels in any patch that most change the output of the network.  We now present extensive empirical evidence corroborating our ansatz.  We subsequently present supporting theoretical evidence. 

\begin{figure}[t]
    \centering
    \includegraphics[width=1\textwidth]{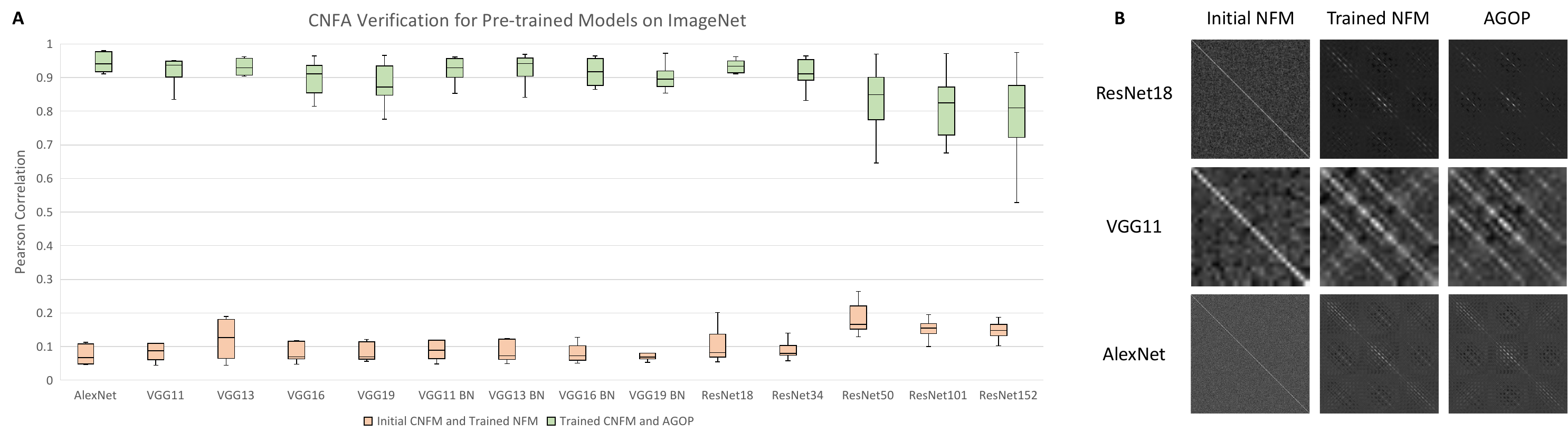}    
    \caption{\textbf{A.} Correlation between initial CNFM and trained CNFM (red) and trained CNFM with AGOP (green) for convolutional layers in VGG, AlexNet, and ResNet on ImageNet ($224 \times 224$ resolution color images). \textbf{B.} Initial CNFM, trained CNFM, and AGOP matrices for the first convolutional layer of ResNet18, VGG11, and AlexNet on ImageNet.}
    \label{fig: SOTA Conv NFA}
\end{figure}

\subsection{Empirical evidence for CNFA}

We now provide empirical evidence for the ansatz by computing the correlation between CNFMs and the AGOP for each convolutional layer in various CNNs.  We provide three lines of evidence by computing correlations for the following models: (1) AlexNet~\cite{AlexNet}, all VGGs~\cite{VGG}, and all ResNet~\cite{ResNet} models pre-trained on ImageNet~\cite{ImageNet} ; (2) SimpleNet models~\cite{SimpleNet} trained on SVHN~\cite{SVHN}, GTSRB~\cite{GTSRB}, CIFAR10~\cite{CIFAR10}, CIFAR100, and ImageNet32~\cite{ImageNet32}; and (3) shallow CNNs across $10$ standard computer vision datasets from PyTorch upon varying pooling and patch size of convolution operations.  The first set of experiments provides evidence for the ansatz in large-scale state-of-the-art models on ImageNet. The second set provides evidence for the ansatz across standard computer vision datasets.  The last set provides evidence for the ansatz holding across architecture choices.   

\paragraph{CNFA verification for pre-trained state-of-the-art models on ImageNet.} We begin by providing evidence for the ansatz on pre-trained state-of-the-art models on ImageNet.  In Fig.~\ref{fig: SOTA Conv NFA}, we present these correlations for AlexNet, all VGG models and all ResNet models pre-trained on ImageNet, which are available for download from the PyTorch library~\cite{PyTorch}.\footnote{We evaluate all correlations between AGOP and CNFMs for all convolutional layers of AlexNet and all VGGs.  To simplify computation on ResNets, we evaluate correlations between AGOP and CNFMs for the first layer in each BasicBlock and each Bottleneck, as defined in PyTorch.  We note that for ResNet152, this computation involves computing correlation between matrices in $50$ Bottleneck blocks.}  As a control, we verify that weights at the end of training are far from initialization (see the red bars in Fig.~\ref{fig: SOTA Conv NFA}A).  Note that despite the complexity involved in training these models (e.g., batch normalization, skip connections, custom optimization procedures, data augmentation) the Pearson correlation between the AGOP and CNFMs are remarkably high ($> .9$ for each layer of AlexNet and VGG13).  In Fig.~\ref{fig: SOTA Conv NFA}B, we additionally visualize the AGOP and CNFM for the first convolutional layer in AlexNet, VGG11, and ResNet18 to demonstrate the qualitative similarity between these matrices.  In addition, in Appendix Fig.~\ref{appendix fig: CNFA Init vs Final}, we verify that these correlations are lower at initialization than at the end of training indicating that the ansatz is, in fact, a consequence of training.

\paragraph{CNFA verification for SimpleNet on CIFAR10, CIFAR100, ImageNet32, SVHN, GTSRB.}  To verify the ansatz on other datasets, we also trained the SimpleNet model on five datasets including CIFAR10/100, ImageNet32, SVHN, and GTSRB. We note SimpleNet had achieved state-of-the-art results on several of these tasks at the time of its release (e.g., $>95\%$ test accuracy on CIFAR10).  We train SimpleNet models using the same optimization procedure provided from~\cite{SimpleNet} (i.e., Adadelta~\cite{AdaDelta} with weight decay and manual learning rate scheduling).  We use a small initialization scheme of normally distributed weights with a standard deviation of $10^{-4}$ for convolutional layers. We note that we were able to recover high test accuracies across all datasets consistent with the results from~\cite{SimpleNet} (see test accuracies for these trained SimpleNet models in Appendix Fig.~\ref{appendix fig: Simplenet CNFA}).  As shown in Appendix Fig.~\ref{appendix fig: Simplenet CNFA}, we observe consistently high correlation between AGOPs and CNFMs across layers of SimpleNet.

\paragraph{CNFA is robust to hyperparameter choices.} We lastly study the effect of patch size and architecture choices on the CNFA for networks trained using the Adam optimizer~\cite{Adam}. We generally observe that larger patch sizes slightly reduce the correlation between AGOP and CNFMs, and that max pooling layers (in contrast to no pooling or average pooling) lead to higher correlation (Appendix Fig.~\ref{appendix fig: Conv NFA}). Interestingly, these results indicate that the choices used in state-of-the-art CNNs (max pooling layers and patch size of $3$) are consistent with those that lead to highest correlation between AGOP and CNFMs.

\begin{figure}[!t]
    \centering
    \includegraphics[width=\textwidth]{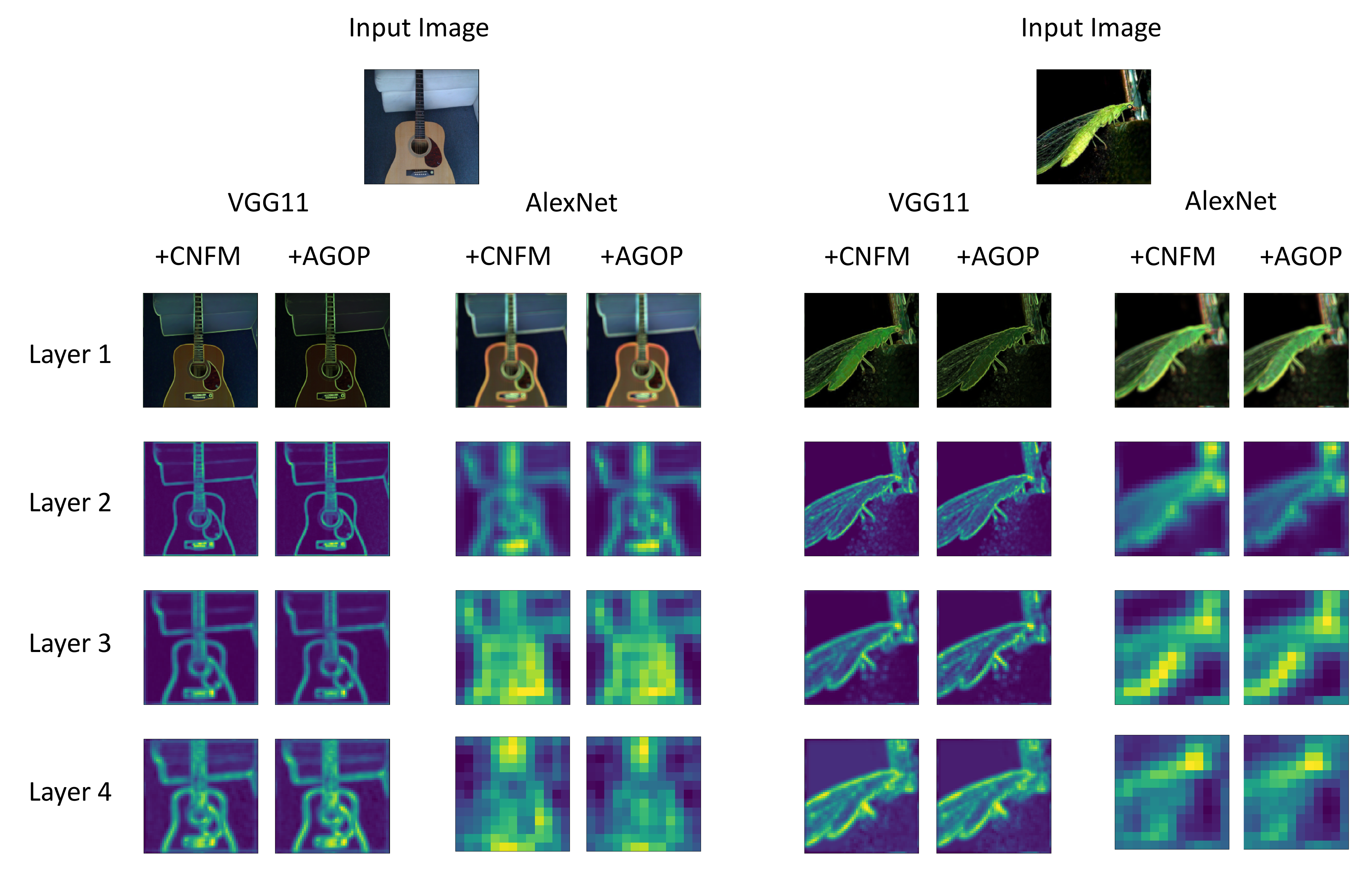}
    \caption{Comparison of features extracted by CNFMs and AGOPs across layers of VGG11 and AlexNet for two input images.  These visualizations provide further supporting evidence that the CNFMs and AGOPs of early layers are performing an operation akin to edge detection.}
    \label{fig: Fig 2}
\end{figure}

\subsection{Visualizing features captured by CNFM and AGOP}

We now visualize how the CNFM operates on patches of images to select features and demonstrate that AGOP over patches captures similar features.  Both the CNFM and AGOP yield an operator on patches of images.  Thus, to visualize how these matrices select features, we expand input images into individual patches, then apply either the CNFM or the AGOP to each patch.  We then reduce the expanded image back to its original size by taking the norm over the spatial dimensions of each expanded patch.  Formally, the value for each coordinate $(i,j) \in P_{\ell-1} \times Q_{\ell-1}$ is replaced with $\|M_{\ell}^{\frac{1}{2}} h_{\ell-1}(X)[i,j]\|$ where $M_{\ell} := W_{\ell}^T W_{\ell}$.   Our visualization reflects the magnitude of the patch in the image of the patch transformation. For example, if $M_{\ell}$ is an edge detector, then $\|M_{\ell}^{\frac{1}{2}} h_{\ell-1}(X)[i,j]\|$ will be large, if and only if the patch centered at coordinate $(i,j)$ contains an edge.


This visualization technique emerges naturally from the convolution operation in CNNs, where a post-activation hidden unit is generated by applying a filter to each patch independently of the others. Further, this visualization characterizes how a trained CNN extracts features across patches of any image. This is in contrast to visualization techniques based on saliency maps~\cite{SaliencyMap, CAM, DeepLIFT, Grad-CAM}, which consider gradients with respect to an entire input image and for a single sample.  

In addition to the high correlation between AGOP and CNFMs in the previous section, in Fig.~\ref{fig: Fig 2}, we observe that the AGOP and CNFMs transform input images similarly at any given layer of the CNN.  For $224 \times 224$ images from ImageNet, CNFMs and AGOPs extracted from a pre-trained VGG11 model both emphasize objects and their edges in the image.  We note these visualizations corroborate hypotheses from prior work that the first layer weights of deep CNNs learn an operator corresponding to edge detection~\cite{ZeilerVisualizing}.  Moreover, our results imply that the mathematical origin of edge detectors in convolutional neural networks is the average gradient outer product.  In the following section, we will corroborate this claim by demonstrating that such edge detectors can be recovered without the use of any neural network through estimating the average gradient outer product of convolutional kernel machines.

\subsection{Supporting Theoretical Evidence for CNFA}

The following theorem (proof in Appendix~\ref{appendix: Theoretical Evidence}) proves the ansatz for general convolutional networks after 1 step of full-batch gradient descent.
\begin{theorem}
\label{thm:one_step}
Let $f$ denote a function that operates on $m$ patches of size $q$, i.e., let $f(v_1, v_2, \ldots, v_m): \mathbb{R}^{q} \times \ldots \times \mathbb{R}^{q} \to \mathbb{R}$ with $f(v_1, v_2, \ldots, v_m) = g(Wv_1, Wv_2, \ldots, Wv_m)$ where $W \in \mathbb{R}^{k \times q}$ and $g(z_1, \ldots, z_m): \mathbb{R}^{k} \times \ldots \times \mathbb{R}^{k} \to \mathbb{R}$.  Assume $g(\mathbf{0}) = 0$ and $\frac{ \partial g(\mathbf{0})}{\partial z_{\ell}} = \frac{\partial g(\mathbf{0})}{\partial z_{\ell'}} \neq 0$ for all $\ell, \ell' \in [m]$.  If $W$ is trained for one step of gradient descent with mean squared loss on data $\{((v_1^{(p)}, \ldots v_m^{(p)}), y_p)\}_{p=1}^{n}$ from initialization $W^{(0)} = \mathbf{0}$, then for the point $(u_1, \ldots, u_m)$: 
\begin{align}
\label{eq: general model one step}
{W^{(1)}}^T W^{(1)} \propto \sum_{r=1}^{m} \frac{\partial f^{(1)}(u_1, \ldots, u_m)}{\partial v_r} \frac{\partial f^{(1)}(u_1, \ldots, u_m)}{\partial v_r}^T ~;
\end{align}
where $f^{(1)}(v_1, v_2, \ldots v_m):= g(W^{(1)}v_1, W^{(1)}v_2, \ldots, W^{(1)}v_m)$.
\end{theorem}
We note the assumptions of Theorem~\ref{thm:one_step} hold for several types of convolutional networks.  As a simple example, the assumptions hold for convolutional networks with activation function $\phi$ satisfying $\phi(0) = 0$ and $\phi'(0) \neq 0$ (e.g., tanh activation) with remaining layers initialized as constant matrices.  Furthermore, we note that while the above theorem is stated for the first layer of a convolutional network, the same proof strategy applies for deeper layers by considering the subnetwork $G_\ell(x)$. 
\section{CNFA as a general mechanism for convolutional feature learning}

We now show that the CNFA allows us to introduce a feature learning mechanism in any machine learning model on patches to capture features akin to those of convolutional networks.  Given recent work connecting neural networks to kernel machines~\cite{NTKJacot}, we focus on convolutional kernels given by the Convolutional Neural Tangent Kernel (CNTK)~\cite{CNTKArora} as our candidate model class.  Intuitively, these models can be thought of as combining kernels evaluated across pairs of patches in images.  While such models have achieved impressive performance~\cite{BiettiDeepConvRep, BiettiKernel, EnchancedCNTK, MyrtleKernels, SimpleFastFlexibleMatrixCompletion, CNTKMillonsExamples}, these models do not automatically learn features from data unlike CNNs.  Thus, as demonstrated in prior work~\cite{karp2021local, PreetumLimitations}, there are tasks where CNTKs are significantly outperformed by corresponding CNNs.

A major consequence of the CNFA is that we can now enable feature learning in CNTKs by leveraging the AGOP over patches.  In particular, we can first solve kernel regression with the CNTK and then use the AGOP of the trained predictor over patches of images to learn features.  We call our method the \textit{Convolutional Recursive Feature Machine (ConvRFM)}, as it is the convolutional variant of the original RFM \cite{radhakrishnan2022feature}.  We will demonstrate that ConvRFM accurately captures first layer feature learning in CNNs and can recover edge detectors as features when trained on standard image classification datasets.  To account for deep convolutional feature learning, we extend ConvRFM to Deep ConvRFMs by sequentially learning features in a manner similar to layerwise training in CNNs.  We show that Deep ConvRFM: (1) improves performance of CNTKs on local signal adaptivity tasks considered in~\cite{karp2021local} ; and (2) improves performance of CNTKs on several image classification tasks.

\begin{algorithm}[!t]
\caption{Convolutional Recursive Feature Machine (ConvRFM)}\label{alg:ConvRFM}
\begin{algorithmic}
\Require $X, y, K_M, T, q$ \Comment{Train data: $(X, y)$, kernel: $K_M$, iters.: $T$, and patch size: $q$}
\Ensure $\alpha, M$ \Comment{Solution to kernel regression: $\alpha$, and feature matrix: $M$}
\State $M = I_{cq^2}$ \Comment{Initialize $M$ to be the identity matrix of size $cq^2 \times cq^2$}
\For{$t \in T$}
    \State $K_{train} = K_M(X, X)$\Comment{$K_M(X, X)_{i,j} := K_M(x_i, x_j)$} 
    \State $\alpha = yK_{train}^{-1}$
    \State $M = \frac{1}{n} \sum_{x \in X} \sum_{(u,v) \in \mathcal{S}} (\nabla_{x[u,v]} f(x)) (\nabla_{x[u,v]} f(x))^T$\Comment{$f(x) = \alpha K_M(X, x)$} 
\EndFor
\end{algorithmic}
\end{algorithm}

\subsection{Convolutional Recursive Feature Machine (ConvRFM)}

We present the algorithm for ConvRFM in Algorithm~\ref{alg:ConvRFM}. The ConvRFM algorithm recursively learns a feature extractor on patches of a given image by implementing the AGOP across patches of training data. Namely, the ConvRFM first builds a predictor with a fixed convolutional kernel.  Then, we compute the AGOP of the trained predictor with respect to image patches, which we denote as the \textit{feature matrix}, $M$.  Lastly, we transform image patches with $M$ and then repeat the previous steps.  We provide a concrete example of this algorithm for the convolutional neural network Gausssian process (CNNGP)~\cite{CosineKernel, JaschaOrderedChaoticPhase} of a one hidden layer convolutional network with fully connected last layer operating on black and white images below.

The CNNGP of a one hidden layer convolutional network with fully connected last layer, activation $\phi$, and filter size $q$ is given by 
\begin{align*}
    K(x, z) = \frac{1}{PQ} \sum_{i=1}^{P}\sum_{j=1}^{Q} \check{\phi}(x[i, j]^T z[i, j], \|x[i, j]\|, \|z[i, j]\|)~;
\end{align*}
where $x, z \in \mathbb{R}^{P \times Q}$, $x[i,j] \in \mathbb{R}^{q^2}$ denotes the vectorized $q \times q$ patch of $x$ centered at coordinate $(i,j)$, and $\check{\phi}(a^Tb, \|a\|, \|b\|)$ denotes the dual activation~\cite{DualActivation} of $\phi$.  For the case of ReLU activation, this dual activation has a well known form~\cite{CosineKernel} and is given by
\begin{align*}
    \check{\phi}(a^Tb, \|a\|, \|b\|) = \frac{1}{\pi} \left( a^T b \left(\pi - \arccos \left( \frac{a^T b}{\|a\|\|b\|} \right) \right) + \sqrt{\|a\|^2 \|b\|^2 - a^T b} \right)~.
\end{align*}
 In ConvRFM, we modify the inner product in the kernel above to be a Mahalanobis inner product, constructing kernels of the form
\begin{align*}
    K_M(x, z) := \frac{1}{PQ}  \sum_{i=1}^{P}\sum_{j=1}^{Q} \check{\phi}(x[i, j]^T M z[i, j], x[i, j]^T M x[i, j], z[i, j]^T M z[i,j])~;
\end{align*}
where $M$ is a learned positive semi-definite matrix. In particular, $M$ is updated as the AGOP of the estimator constructed by solving kernel regression with $K_M$.  In our experiments, we analyze performance when replacing $\check{\phi}$ with the Mahanolobis Laplace kernel used in~\cite{radhakrishnan2022feature} and with the CNTK of a deep convolutional ReLU network with fully connected last layer.  We will make clear our choice of $\check{\phi}$ by denoting our method as CNTK-ConvRFM or Laplace-ConvRFM.

\begin{figure}[t]
    \centering
    \includegraphics[width=1\textwidth]{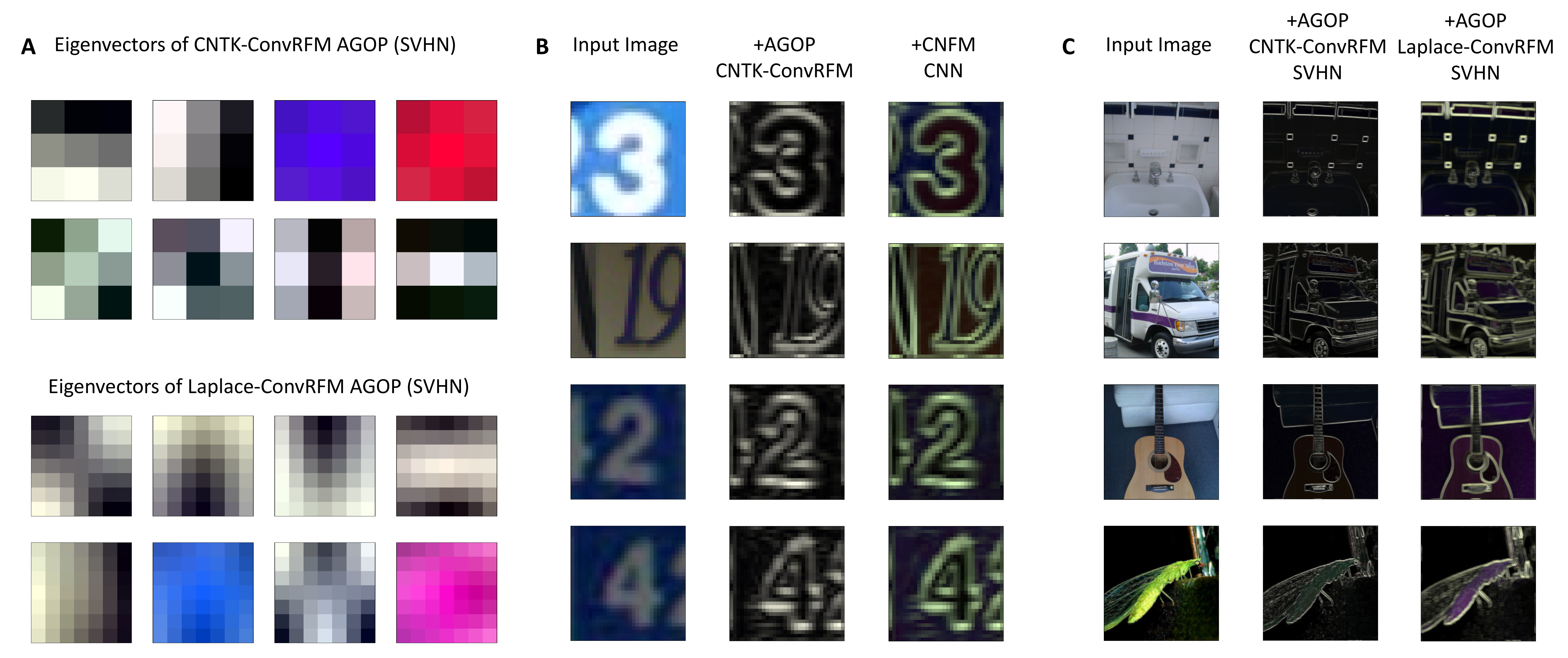}
    \caption{Features extractors learned by ConvRFM using CNTK (CNTK-ConvRFM) and Laplace kernel (Laplace-ConvRFM), which appear to operate as universal edge detectors.  \textbf{A.} Top 8 eigenvectors of CNTK-ConvRFM and Laplace-ConvRFM trained on SVHN.  We use $3 \times 3$ patches for CNTK-ConvRFM and $7 \times 7$ patches for Laplace-ConvRFM. \textbf{B.} Comparison of patch operators learned by CNTK-ConvRFM (given by the AGOP taken with respect to patches) and CNNs (given by the CNFM). \textbf{C.} Applying patch-based AGOP operators from ConvRFMs trained on SVHN to images from ImageNet.  }
    \label{fig: Conv-RFM feature learning}
\end{figure}

\paragraph{ConvRFM captures first layer features of convolutional neural networks.}

We now demonstrate that ConvRFM recovers features similar to those learned by first layers of CNNs.  In Fig.~\ref{fig: Conv-RFM feature learning}A, we visualize the top eigenvectors of the feature matrix of CNTK-ConvRFM (filter size $3 \times 3$) and Laplace-ConvRFM (filter size $7 \times 7$) trained on SVHN.  Training details for all methods are presented in Appendix~\ref{appendix: experimental details}.  We observe that these top eigenvectors resemble edge detectors~\cite{DigitalImageProcessingGonzales}.  In Fig.~\ref{fig: Conv-RFM feature learning}B, we visualize how the feature matrix of the CNTK-ConvRFM and the CNFM of the corresponding finite width CNN trained on SVHN transform SVHN images.  Even though both operators arise from vastly different training procedures (solving kernel regression vs. training a CNN), we observe that both operators appear to extract similar features (corresponding to edges of digits) from SVHN images.  We provide additional evidence for similarity between ConvRFM and CNN features in Appendix Fig.~\ref{appendix fig: Conv RFM vs. CNN Features}.  To demonstrate further evidence of the universality of edge detector features arising from AGOP of CNTK-ConvRFM and Laplace-ConvRFM, we analyze how these AGOPs transform arbitrary images.  In particular, in Fig.~\ref{fig: Conv-RFM feature learning}C, we apply these operators extracted from models trained on SVHN to images on ImageNet.  We again observe that these operators remarkably extract edges from corresponding ImageNet images, which are of vastly different resolution ($224 \times 224$ instead of $32 \times 32$) and contain vastly different objects.  Such experiments provide conclusive evidence that AGOP with respect to patches of convolutional kernels recovers features akin to edge detectors.  We present further experiments demonstrating emergence of edge detectors from convolutional kernels trained on CIFAR10 and GTSRB in Appendix Figs.~\ref{appendix fig: ConvRFM Edge Detection} and \ref{appendix fig: CIFAR10 GTSRB ConvRFM Features}. In particular, the eigenvectors of the AGOP often resemble Gabor filters with different orientations. In Figure~\ref{appendix fig: ConvRFM Edge Detection}, we see that horizontally, vertically, and diagonally aligned eigenvectors identify edges of the same alignment.

\begin{algorithm}[!b]
\caption{Deep Convolutional Recursive Feature Machine (Deep ConvRFM)}\label{alg:DeepConvRFM}
\begin{algorithmic}
\Require $X, y, \{K_\ell\}_\ell, T, L, q, k$ 
\Comment{kernels: $\{K_\ell\}_\ell$, depth: $L$, channels: $k$}
\Ensure $\alpha_L, \{M_\ell\}_{\ell=1}^L$ 
\State $X_1 = X$ \Comment{Initialize embedding}
\For{$\ell \in L$}
    \State $\alpha_\ell, M_\ell = \text{ConvRFM}(X_\ell, y, K_\ell, T, q)$
    \State Sample $k$ filters $W_{\ell,k'} \sim \mathcal{N}(0,M_\ell)$
    \Comment{$W_\ell \in \Real^{k \times q \times q}$}
    \State $X_{\ell + 1} = \phi(W_\ell * X_\ell)$ \Comment{$\phi$: element-wise non-linearity}
\EndFor
\end{algorithmic}
\end{algorithm}

\subsection{Deep feature learning with Deep ConvRFM}
ConvRFM is capable of only extracting features by linearly transforming patches of input images, which is analogous to extracting such features using the first layer of a CNN.  In contrast, the CNFA implies that deep convolutional networks are capable of learning features in intermediate layers.  To enable deep feature learning, we introduce Deep ConvRFM (see Algorithm~\ref{alg:DeepConvRFM}) by sequentially learning features with AGOP in a manner similar to layerwise training in CNNs.  In particular, Deep ConvRFM iterates the following steps: 
\begin{enumerate}
    \item Construct a predictor, $\hat{f}$, by training a convolutional kernel machine with kernel $K_M$.     
    \item Update $M$ to be the AGOP with respect to patches of the trained predictor. 
    \item Transform the data, $x$, with random features given by $\phi(Wx)$ where $W$ denotes a set of convolutional filters with weights sampled according to $\mathcal{N}(0, M)$ and $\phi$ is a nonlinearity.    
\end{enumerate}

Note that while we utilize random features and sample convolutional filters in Deep ConvRFM, we never utilize backpropgation to learn features or train models.  Features are learned via the AGOP and models are trained by solving kernel regression, which is a convex optimization problem.  For the base kernel for Deep ConvRFM, we utilize the deep CNTK~\cite{CNTKArora} as implemented in the Neural Tangents library~\cite{NeuralTangentsGoogle}.\footnote{In order to take gradient with respect to patches using Neural Tangents, we used a workaround that involved expanding images into their patch representations.  This workaround unfortunately leads to heavy memory utilization, which limited our analysis of Deep ConvRFMs.}

\begin{figure}[!t]
    \centering
    \includegraphics[width=1\textwidth]{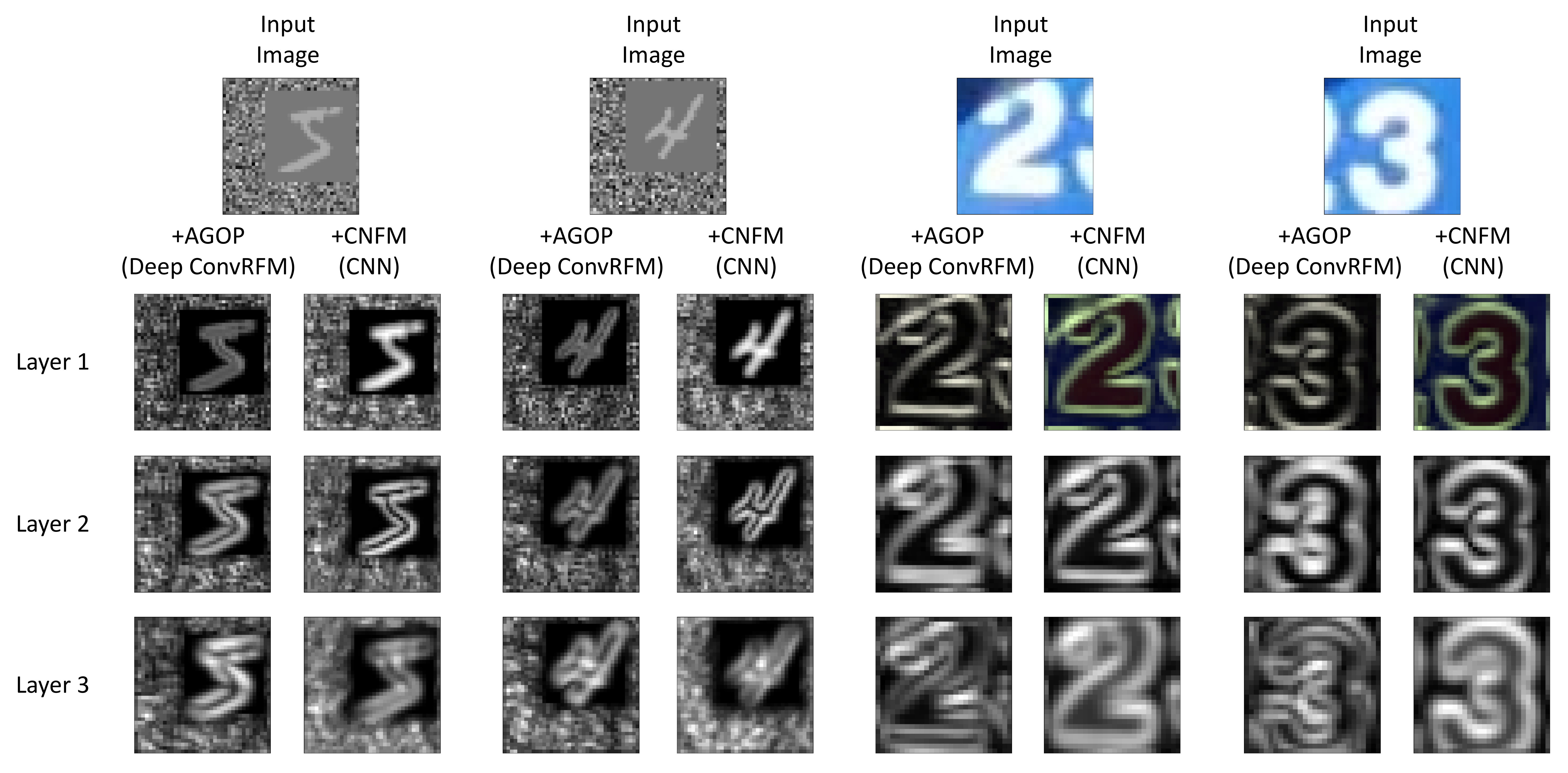}
    \caption{Visualizations of features for each layer of Deep ConvRFM and the corresponding CNN on SVHN and the noisy digits task from~\cite{karp2021local}.}
    \label{fig: Deep FL}
\end{figure}

\paragraph{Deep ConvRFM learns similar features to deep CNNs.}
We now present evidence that Deep ConvRFMs learn similar features to those learned by deep CNNs.  We analyze features learned by deep ConvRFM and the corresponding CNN on the local signal adaptivity synthetic tasks from~\cite{karp2021local} and SVHN.  For the synthetic task from~\cite{karp2021local}, we consider classification of MNIST digits embedded in a larger image of i.i.d. Gaussian noise.  Dataset and training details are presented in Appendix~\ref{appendix: experimental details}.  In Fig.~\ref{fig: Deep FL}, we observe that AGOPs at each layer of Deep ConvRFM and and CNFMs at each layer of the corresponding CNN transform examples from both datasets similarly.  

\paragraph{Deep ConvRFM overcomes limitations of convolutional kernels.} In the work~\cite{karp2021local} the authors posited local signal adaptivity,  the ability to suppress noise and amplify signal in images, as a potential explanation for the superiority of convolutional neural networks over convolutional kernels.  As supporting evidence, \cite{karp2021local} demonstrated that convolutional networks generalized far better than convolutional kernels on image classification tasks in which images were embedded in a noisy background.  We now demonstrate that by incorporating feature learning through patch-AGOPs, Deep ConvRFM exhibits local signal adaptivity on the tasks considered in~\cite{karp2021local} and thus, similar to CNNs, yield significantly improved performance over convolutional kernels.  In particular, we begin by comparing performance of CNTK, Conv RFM, Deep ConvRFM, and corresponding CNNs on the following two image classification tasks from~\cite{karp2021local}: (1) images of black and white horizontal bars placed in a random position on larger images of Gaussian noise ; (2) MNIST images placed in a random position on larger images of Gausssian noise.  The work~\cite{karp2021local} demonstrated that CNNs, unlike CNTK, could learn to threshold the background noise and amplify the signal in these tasks thus far outperforming CNTKs when the amount of background noise was large.  In Fig.~\ref{fig: Toy Tasks}, we demonstrate that for these tasks CNNs, ConvRFMs, and Deep ConvRFMs all extract local signals and dim background noise through the AGOP, and thus far outperform CNTKs. Moreover, we observe that Deep ConvRFMs can provide up to a 5\% improvement in performance over ConvRFM on the synthetic MNIST task, indicating a benefit to deep feature learning.

\begin{figure}[!t]
    \centering
    \includegraphics[width=1\textwidth]{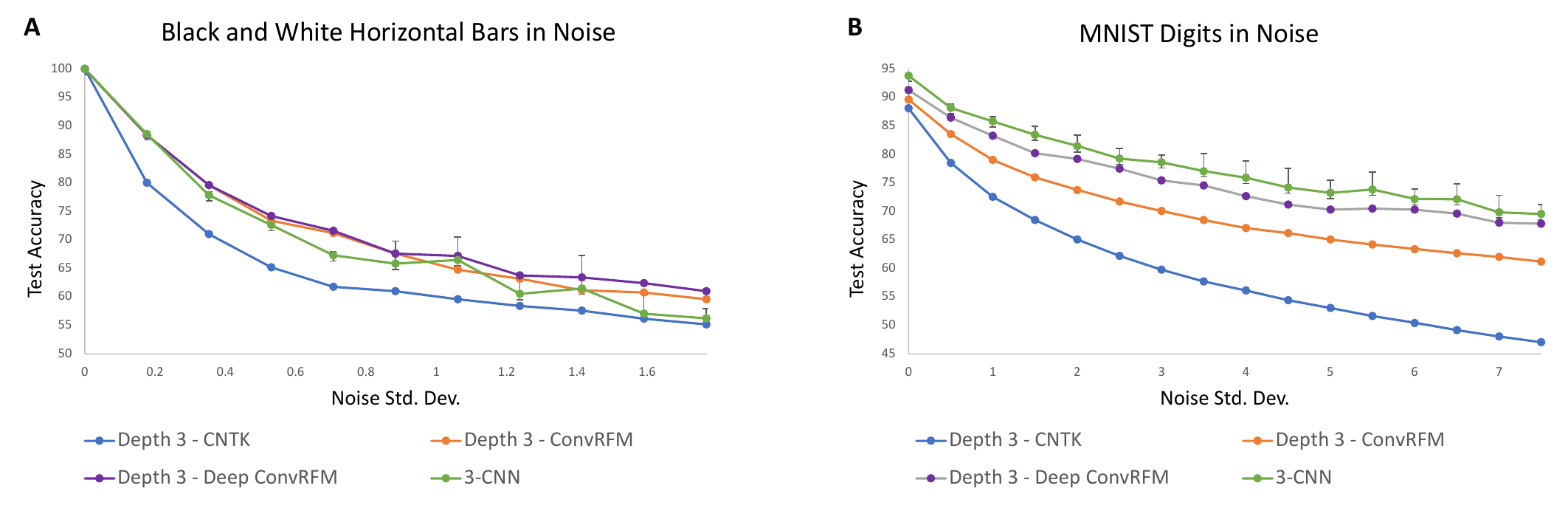}
    \caption{Test accuracy of CNTK, ConvRFM, Deep ConvRFM, and the corresponding CNN on local signal adaptivity tasks from~\cite{karp2021local} as a function of noise level. \textbf{A.} Identifying black and white bars in noisy images. \textbf{B.} MNIST digits placed randomly in noisy background image.}
    \label{fig: Toy Tasks}
\end{figure}

\paragraph{Benefit of deep feature learning on real-world image classification tasks.}  Lastly, we analyze performance of CNTK, ConvRFM, Deep ConvRFM, and the corresponding three convolutional layer CNN on standard image classification datasets available for download from PyTorch.  Consistent with our observations for synthetic tasks from~\cite{karp2021local}, we observe in Fig.~\ref{fig: CRFM Improvement}A that ConvRFM and Deep ConvRFM provide an improvement over CNTK across almost all tasks.  Moreover, we observe that ConvRFM and Deep ConvRFM outperform CNTKs consistently when the corresponding CNN outperforms the CNTK.  In Fig.~\ref{fig: CRFM Improvement}B, we analyze the impact of deep feature learning by increasing the number of feature learning layers in Deep ConvRFM, i.e., the number of layers for which we utilize the AGOP to learn features.  We observe that adding more layers of feature learning leads to consistent performance boost in the local signal adaptivity tasks from~\cite{karp2021local} and on select datasets such as SVHN and EMNIST~\cite{EMNIST}.

\begin{figure}[!t]
    \centering
    \includegraphics[width=1\textwidth]{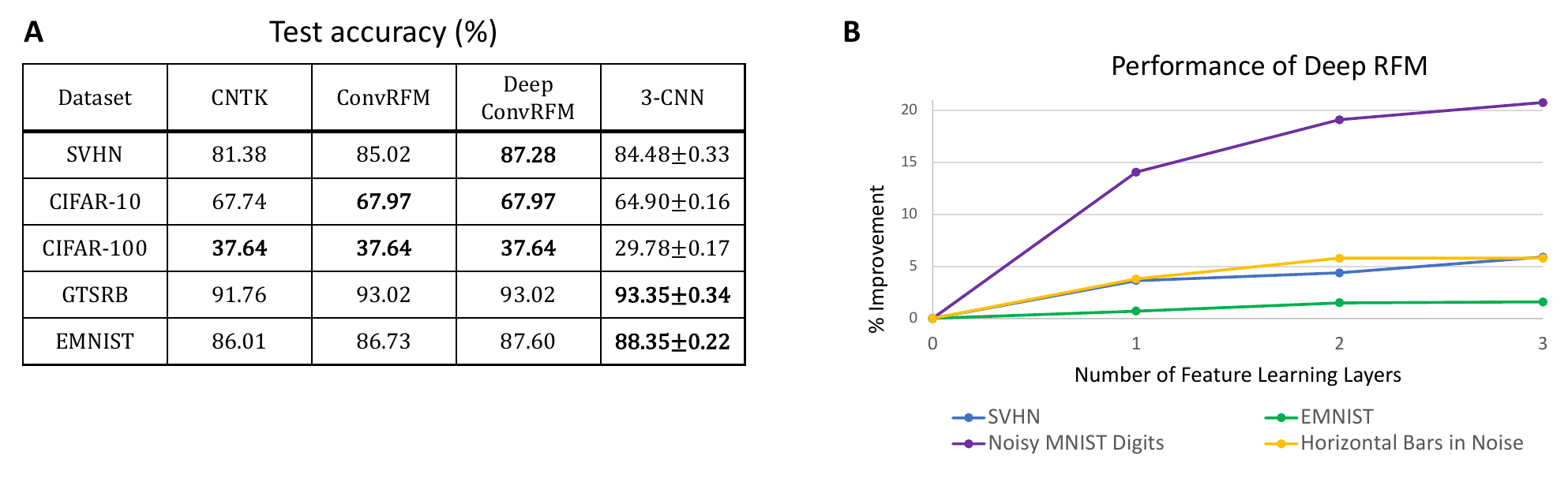}
    \caption{\textbf{A.} Performance comparison of Deep ConvRFM with the corresponding CNTK and CNN on benchmark image classification datasets from PyTorch. \textbf{B.} Effect of number of feature learning layers on Deep ConvRFM performance.}
    \label{fig: CRFM Improvement}
\end{figure}

\section{Discussion}

In this work, we identified a mathematical mechanism of feature learning in deep convolutional networks, which we posited as the Convolutional Neural Feature Ansatz (CNFA).  Namely, the ansatz stated that features selected by convolutional networks, given by empirical covariance matrices of filters at any given layer, can be recovered by computing the average gradient outer product (AGOP) of the trained network with respect to image patches.  We presented  empirical and theoretical evidence for the ansatz.  Notably, we showed that convolutional filter covariances of neural networks pre-trained on ImageNet (AlexNet, VGG, ResNet) are highly correlated with AGOP with respect to patches (in many cases, Pearson correlation $> .9$).  Since the AGOP with respect to patches can be computed on any function operating on image patches, we could use the AGOP to enable feature learning in any machine learning model operating on image patches.  Thus, building on the RFM algorithm for fully connected networks from~\cite{radhakrishnan2022feature}, we integrated the AGOP to enable deep feature learning in convolutional kernel machines, which could not apriori learn features, and referred to the resulting algorithms as ConvRFM and Deep ConvRFM.  We demonstrated that ConvRFM and Deep ConvRFM recover features similar to those of deep convolutional neural networks, including evidence that features learned by these models can serve as universal edge detectors, akin to features learned in convolutional networks.  Moreover, we demonstrated that ConvRFM and Deep ConvRFM overcome prior limitations of convolutional kernels, including the Convolutional Neural Tangent Kernel (CNTK), such the inability to adapt to localized signals in images~\cite{karp2021local}.  Lastly, we showed a benefit to deep feature learning by demonstrating improvement in performance of Deep ConvRFM over ConvRFM and the CNTK on standard image classification benchmarks.  We now conclude with a discussion of implications of our results and future directions.

\paragraph{Identifying mechanisms driving success of deep learning.}  Understanding the mechanisms driving success of neural networks is an important problem for developing effective, interpretable and safe machine learning models.  The complexities of training deep neural networks such as custom training procedures and layer structures (batch normalization, dropout, residual connections, etc.) can make it difficult to pinpoint overarching principles leading to effectiveness of these models.  The fact that correlation between convolutional neural feature matrices (CNFMs) and AGOPs is high for convolutional networks pre-trained on ImageNet with all of these inherent complexities baked in, provides strong evidence that the connection between AGOP and CNFMs is key to identifying the core principles making these networks successful.  

\paragraph{Emergence of universal edge detectors with average gradient outer product.}  Detecting edges in images is a well-studied task in computer vision and classical approaches involved applying fixed convolutional filters to detect edges in images~\cite{DigitalImageProcessingGonzales, ConvEdgeSurvey, ConvRadioEdgeSurvey}.  For example, AlexNet automatically learned filters in its first convolutional layer that were remarkably similar to Gabor filters~\cite{AlexNetGabor}.  Similarly, there was evidence that other convolutional networks pre-trained on ImageNet learned features akin to edge detection in the first layer~\cite{ZeilerVisualizing}.  Yet, it had been unclear how such filters automatically emerge through training.  We demonstrated that the AGOP with respect to patches of a large class of convolutional models (convolutional neural networks and convolutional kernels) trained on various standard image classification tasks consistently recovered edge detectors (see Fig.~\ref{fig: Fig 2}, Fig.~\ref{fig: Conv-RFM feature learning}A, B).  We further showed the universality of these edge detector features by demonstrating that features learned by ConvRFM on SVHN automatically identified edges in ImageNet images. This strongly suggests that edge detectors emerge from the underlying nature of the task rather than specific properties of architectures. 
Our findings indicate that understanding connections between AGOP and classical edge detection approaches is a promising direction for understanding emergence of features in the first layer of convolutional neural networks and for  identifying simple algorithms to capture deeper convolutional features.   

\paragraph{Reducing computational complexity of convolutional kernels.}  In this work, we provided an approach for enabling feature learning in convolutional kernels by iteratively training convolutional kernel machines and computing AGOP of the trained predictor.  Given that convolutional kernels are able to achieve impressive accuracy on standard datasets without any feature learning~\cite{BiettiDeepConvRep, BiettiKernel, EnchancedCNTK, MyrtleKernels, CNTKMillonsExamples}, these methods have the potential to provide state-of-the-art results upon incorporating feature learning.  Yet, in contrast to the case of classical kernel machines such as those used in~\cite{radhakrishnan2022feature}, evaluating the kernel for an effective CNTK (such as those with Global Average Pooling~\cite{CNTKArora}) can be a far more computationally intensive process than simply training a convolutional neural network.  For example, according to Neural Tangents~\cite{NeuralTangentsGoogle}, the CNTK of a Myrtle kernel~\cite{MyrtleKernels} can take anywhere from $300$ to $500$ GPU hours for CIFAR10.  Given that Deep ConvRFM involves constructing a kernel matrix and computing AGOP to capture features at each layer, reducing the evaluation time of convolutional kernels through strategies such as random feature approximations is key to making these approaches scalable.

\section*{Acknowledgements} 
A.R. is supported by the Eric and Wendy Schmidt Center at the Broad Institute.  We acknowledge support from the National Science Foundation (NSF) and the Simons Foundation for the Collaboration on the Theoretical Foundations of Deep Learning\footnote{\url{https://deepfoundations.ai/}} through awards DMS-2031883 and \#814639 as well as the  TILOS institute (NSF CCF-2112665).  This work used the programs (1) XSEDE (Extreme science and engineering discovery environment)  which is supported by NSF grant numbers ACI-1548562, and (2) ACCESS (Advanced cyberinfrastructure coordination ecosystem: services \& support) which is supported by NSF grants numbers \#2138259, \#2138286, \#2138307, \#2137603, and \#2138296. Specifically, we used the resources from SDSC Expanse GPU compute nodes, and NCSA Delta system, via allocations TG-CIS220009.
  
\section*{Code Availability}
All code is available at \url{https://github.com/aradha/convrfm}. 

\bibliographystyle{abbrv}
\bibliography{aux/references}

\appendix 
\section{Theoretical Evidence for Deep Convolutional Feature Ansatz}
\label{appendix: Theoretical Evidence}
\begin{proof}[Proof of Theorem~\ref{thm:one_step}]
Gradient descent proceeds as follows: 
$$B^{(1)} = B^{(0)} + \eta \sum_{p=1}^{n} \sum_{\ell=1}^{m} \frac{\partial g(B^{(0)}v_1^{(p)}, \ldots, B^{(0)} v_m^{(p)})}{\partial z_{\ell}} (y_p - f(v_1^{(p)}, \ldots , v_m^{(p)})) {v_{\ell}^{(p)}}^T~.$$
Since $B^{(0)} = 0$, $g(\mathbf{0}) = 0$ and $\frac{\partial g(\mathbf{0})}{\partial z_{\ell}} = G$ for fixed nonzero $G \in \mathbb{R}^{k}$, the above expression reduces to: 
$$ B^{(1)} = \eta \sum_{p=1}^{n} \sum_{\ell = 1}^{m} G y_p {v_{\ell}^{(p)}}^T~.$$
Thus, we conclude that
\begin{align*}
{B^{(1)}}^T B^{(1)} &= \eta^2 \sum_{p, p'=1}^{n} \sum_{\ell, \ell' = 1}^{m} y_p y_{p'} {v_{\ell}^{(p)}} G^T G {v_{\ell'}^{(p')}}^T \\
&= \left(\eta^2 m^2 G^T G\right) \left( \sum_{p, p'=1}^{n} \sum_{\ell, \ell' = 1}^{m} y_p y_{p'} {v_{\ell}^{(p)}}  {v_{\ell'}^{(p')}}^T \right)\\
&\propto \sum_{p, p'=1}^{n} \sum_{\ell, \ell' = 1}^{m} y_p y_{p'} {v_{\ell}^{(p)}}  {v_{\ell'}^{(p')}}^T~.
\end{align*}
Now, we finish the proof by showing the right hand side of Eq.~\eqref{eq: general model one step} is proportional to the same quantity above.  First, we have that
$$\frac{\partial f^{(1)}(u_1, \ldots, u_m)}{\partial v_r} = {B^{(1)}}^T \frac{\partial g(B^{(1)}u_1, \ldots, B^{(1)}u_m)}{\partial z_r}~.$$
Thus, letting $u = (u_1, \ldots u_m)$ and $B^{(1)}u = (B^{(1)}u_1, \ldots, B^{(1)}u_m)$, we have that
\begin{align*}
     \sum_{r=1}^{m} \frac{\partial f^{(1)}(u)}{\partial v_r} \frac{\partial f^{(1)}(u)}{\partial v_r}^T &= \sum_{r=1}^{m} {B^{(1)}}^T \frac{\partial g(B^{(1)}u)}{\partial z_r} \frac{\partial g(B^{(1)}u)}{\partial z_r}^T {B^{(1)}}^T \\
     &= \eta^2 \sum_{r=1}^{m} \left( \sum_{p=1}^{n} \sum_{\ell = 1}^{m}  y_p {v_{\ell}^{(p)}} G^T \frac{\partial g(B^{(1)}u)}{\partial z_r} \right) \left( \sum_{p'=1}^{n} \sum_{\ell' = 1}^{m} \frac{\partial g(B^{(1)}u)}{\partial z_r}^T G y_p'{v_{\ell'}^{(p')}}^T  \right) \\
     &= \eta^2 \sum_{p, p'=1}^{n} \sum_{\ell, \ell' = 1}^{m}  y_p y_{p'} {v_{\ell}^{(p)}} {v_{\ell'}^{(p')}}^T  \left( \sum_{r=1}^{m} G^T \frac{\partial g(B^{(1)}u)}{\partial z_r} \frac{\partial g(B^{(1)}u)}{\partial z_r}^T G \right) \\
     &\propto \sum_{p, p'=1}^{n} \sum_{\ell, \ell' = 1}^{m} y_p y_{p'} {v_{\ell}^{(p)}}  {v_{\ell'}^{(p')}}^T~.
\end{align*}
\end{proof}

\section{Experimental Details}
\label{appendix: experimental details}

\paragraph{Neural network comparisons.} For all neural network experiments, we reported the best test accuracy across all epochs. For the Adam experiments, we trained CNNs without bias, with learning rate $10^{-4}$, width $64$, without padding and with minibatch size $128$. For EMNIST, CIFAR-10/100, SVHN, GTSRB the networks were trained for 500 epochs. For the toy datasets, the networks were trained for 25 epochs. For SGD experiments, the setup was identical except the learning rate was $10^{-1}$, and EMNIST, CIFAR-10/100, SVHN, GTSRB were trained for $2000$ epochs, and toy datasets for $100$ epochs.

\paragraph{Visualizations.} For the visualizations in Figs.~\ref{fig: Conv-RFM feature learning}, \ref{appendix fig: CIFAR10 GTSRB ConvRFM Features}, the toy tasks were visualized with $\|M^{1/2}x\|$, while CIFAR and SVHN were visualized with $\|Mx\|$. Further, for the neural networks in CIFAR and SVHN, the initial weight matrices were subtracted before using the CNFMs. For Figure~\ref{fig: Deep FL}, the visualizations were also done with the full $M$ matrix and subtracting the initial weights. The weight matrices were extracted after $250$ epochs. For visualization, the $M$ matrix that gave the best performance of the $5$ iterations of ConvRFM was selected, and the CNN neural feature matrices were extracted at the end of training.

\paragraph{Deep ConvRFM.} For Deep-ConvRFM experiments, we greedily selected the best performing $M$ matrix among $5$ rounds of ConvRFM for each depth. Further, we tuned regularization among $10^{-8}$, $10^{-5}$, $10^{-3}$ and divided the train and test kernel matrices by the maximum value of the train kernel matrix. In all of the experiments, we used the same architecture as the CNN. In particular, we sampled $64$ filters in each layer and removed bias. Further, to ensure that the ConvRFM did not have access to a kernel of additional depth, we reduced the depth of the kernel by $1$ with each layer of Deep ConvRFM. Instead of sampling from $M$, we generated filters sampled from standard normal distribution and applied $M^{1/2}$ to each filter. 

\paragraph{CNFA verification.} For the CNFA verification experiments in Figure~\ref{appendix fig: Conv NFA}, we used the uncentered correlation (commonly-known as cosine similarity) to measure similarity between the AGOP and CNFM. The correlation was averaged over ten datasets: Fine-Grained Visual Classification of Aircraft \cite{AircraftDataset}, PatchCamelyon \cite{PCAM}, CIFAR-10, STL-10~\cite{STL10}, GTSRB, SVHN, Caltech101, DTD \cite{DTD}, QMNIST \cite{QMNIST}, EMNIST. We used zero padding in all layers, a learning rate of $2 \times 10^{-4}$, $500$ epochs of training with the Adam optimizer, and minibatch size of $128$. For datasets with multiple color channels, the RGB color channels were scaled and centered to have means [125.3/255, 123/125, 113.9/225] and standard deviations [63.0/225, 62.1/225, 66.7/225], respectively. Datasets with images larger than 32x32 resolution were resized using PyTorch's resize transform to 32x32 resolution. All layers were initialized from a standard normal distribution. The first layer was initialized with standard deviation $5\times 10^{-3}$, while the remaining layers were sampled with standard deviation $10^{-2}$. 

\paragraph{Convolutional kernel implementation.}

To implement convolutional kernels, we used the Neural Tangents library \cite{NeuralTangentsGoogle}. Mahalanobis kernels were not implemented directly in this library at the time of publication. To implement ConvRFM, we performed the following procedure: (1) unfold each image into all patches (without any padding), (2) applied $M^{1/2}$ to each patch independently, (3) reshaped the images to be 2-dimensional, and (4) set the stride of the first-convolutional layer in the kernel to the patch size, then ran kernel regression. 
\clearpage

\begin{figure}[h]
    \centering
    \includegraphics[width=.8\textwidth]{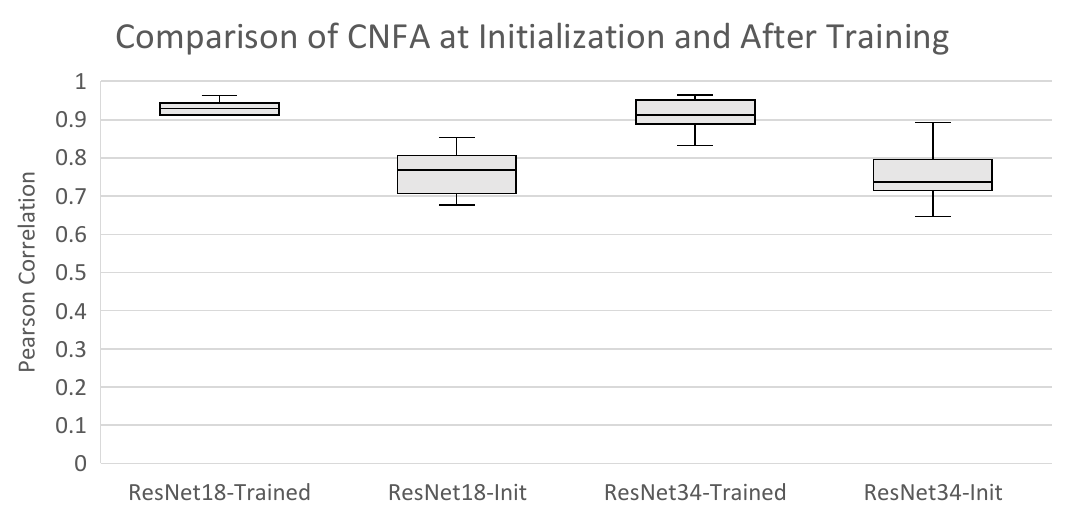}
    \caption{Comparison of correlation between CNFMs and AGOP for randomly initialized ResNets and pre-trained ResNets on ImageNet.}
    \label{appendix fig: CNFA Init vs Final}
\end{figure}

\clearpage

\begin{figure}[h]
    \centering
    \includegraphics[width=\textwidth]{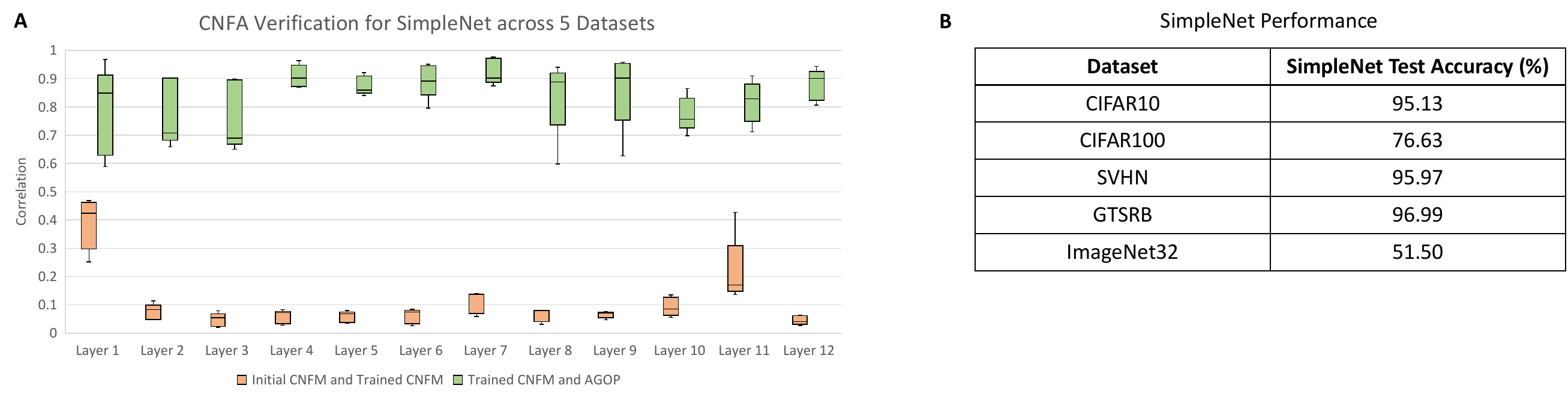}
    \caption{\textbf{A.} Correlation between initial CNFM and trained CNFM (red) and trained CNFM and AGOP (green) for each convolutional layer of SimpleNets trained on 5 datasets.  \textbf{B.} Performance of SimpleNet on the 5 corresponding datasets.}
    \label{appendix fig: Simplenet CNFA}
\end{figure}

\clearpage 

\begin{figure}[h]
    \centering
    \includegraphics[width=1\textwidth]{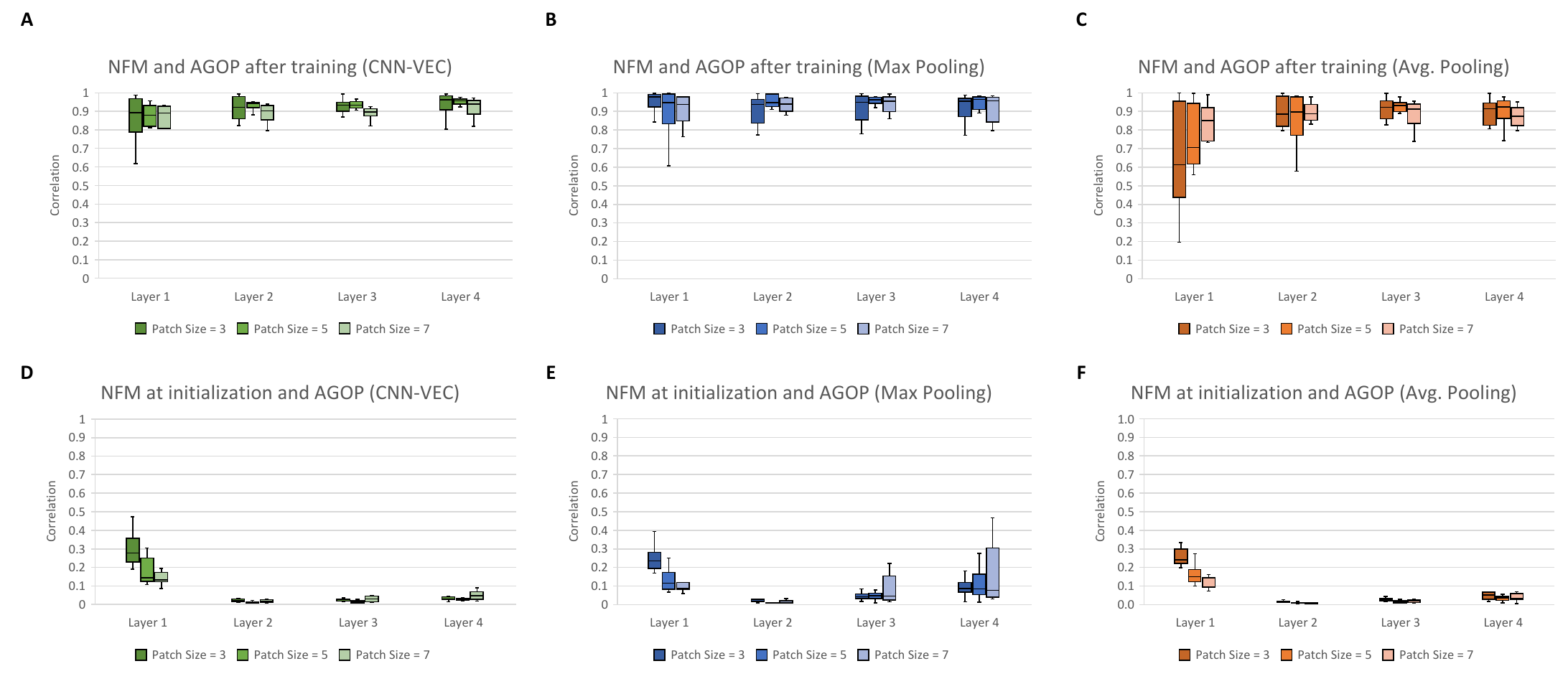}
    \caption{Correlation (cosine similarity) of NFM and AGOP across patch sizes and pooling types on ten datasets. \textbf{(A)} CNN-VEC, trained NFM and AGOP, \textbf{(B)} Max pooling, trained NFM and AGOP, \textbf{(C)} Average pooling, trained NFM and AGOP, \textbf{(D)} CNN-VEC, initial NFM and AGOP, \textbf{(E)} Max pooling, initial NFM and AGOP, \textbf{(F)} Average pooling, initial NFM and AGOP.}
    \label{appendix fig: Conv NFA}
\end{figure}

\clearpage

\begin{figure}[h]
    \centering
    \includegraphics[width=1\textwidth]{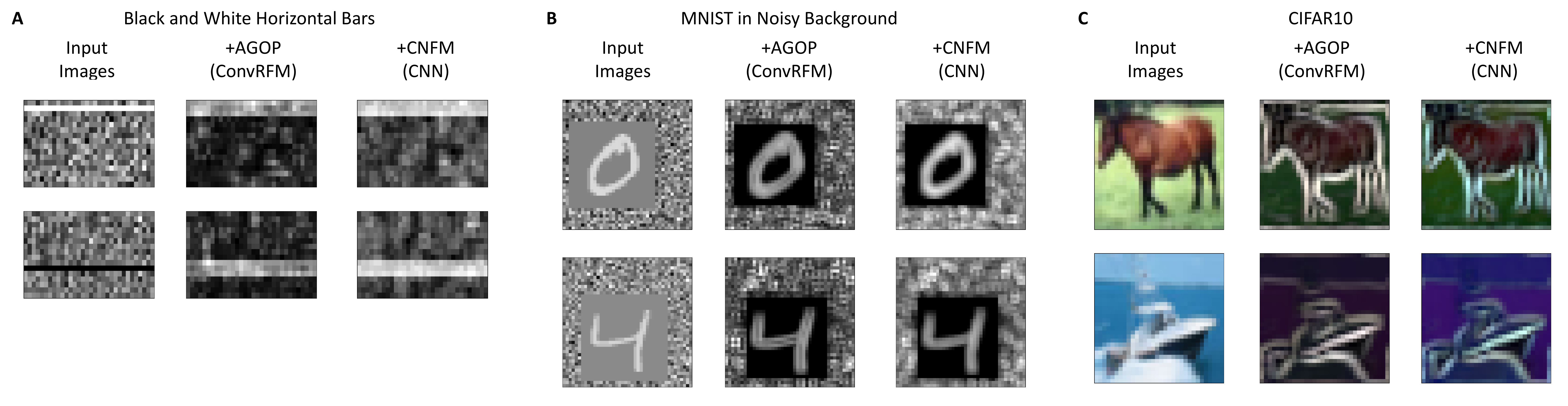}
    \caption{Comparison of feature extraction performed by patch-AGOPs and CNFMs from ConvRFM and the corresponding CNN. \textbf{A.} Visualizations for models trained on horizontal bars in noisy backgrounds (the toy task from~\cite{karp2021local}).  \textbf{B.} Visualizations for models trained on MNIST images in noisy backgrounds, which was also the task considered in~\cite{karp2021local}.  \textbf{C.} Visualizations for models trained on CIFAR10.}
    \label{appendix fig: Conv RFM vs. CNN Features}
\end{figure}

\clearpage

\begin{figure}[h]
    \centering
    \includegraphics[width=1\textwidth]{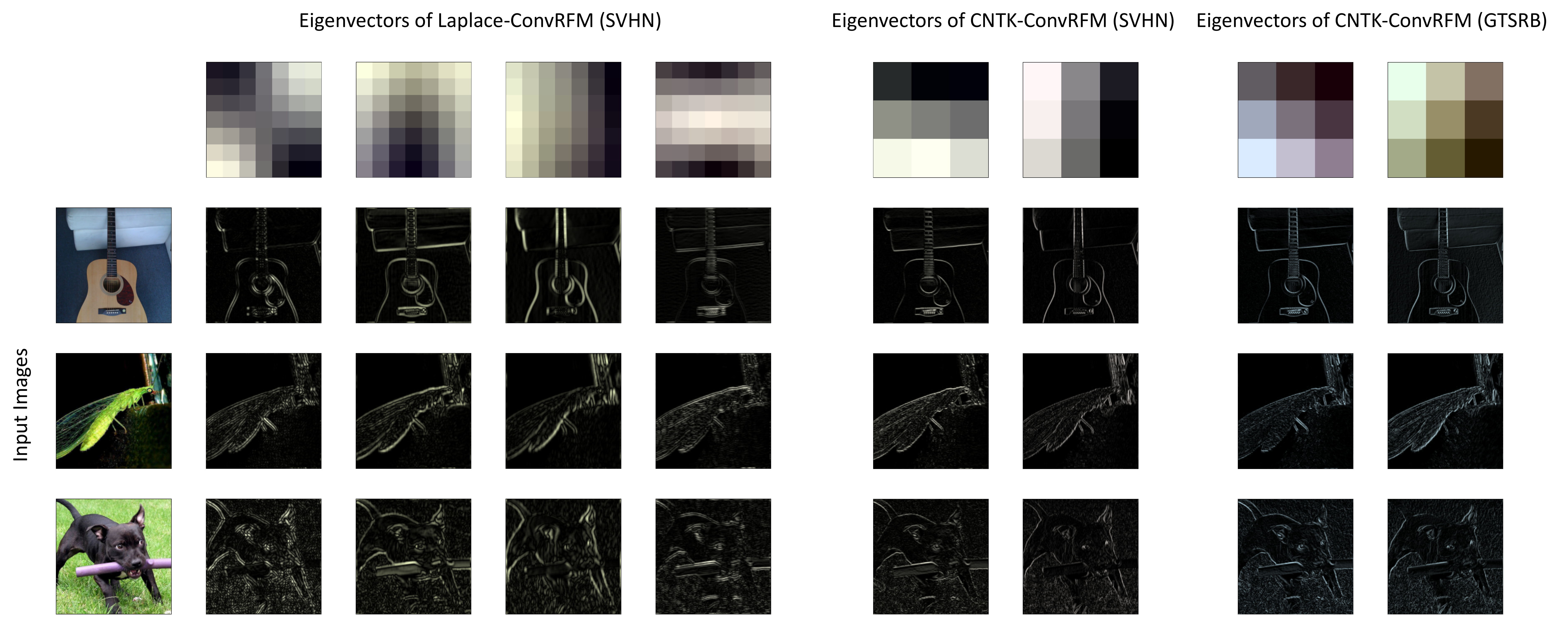}
    \caption{Top eigenvectors of patch-AGOP from Laplace-ConvRFM and CNTK-ConvRFM trained on standard image classification datasets (SVHN and GTSRB) act as universal edge detectors of different orientations.  We visualize images after applying each top eigenvector to image patches of ImageNet images.}
    \label{appendix fig: ConvRFM Edge Detection}
\end{figure}

\clearpage

\begin{figure}[h]
    \centering
    \includegraphics[width=1\textwidth]{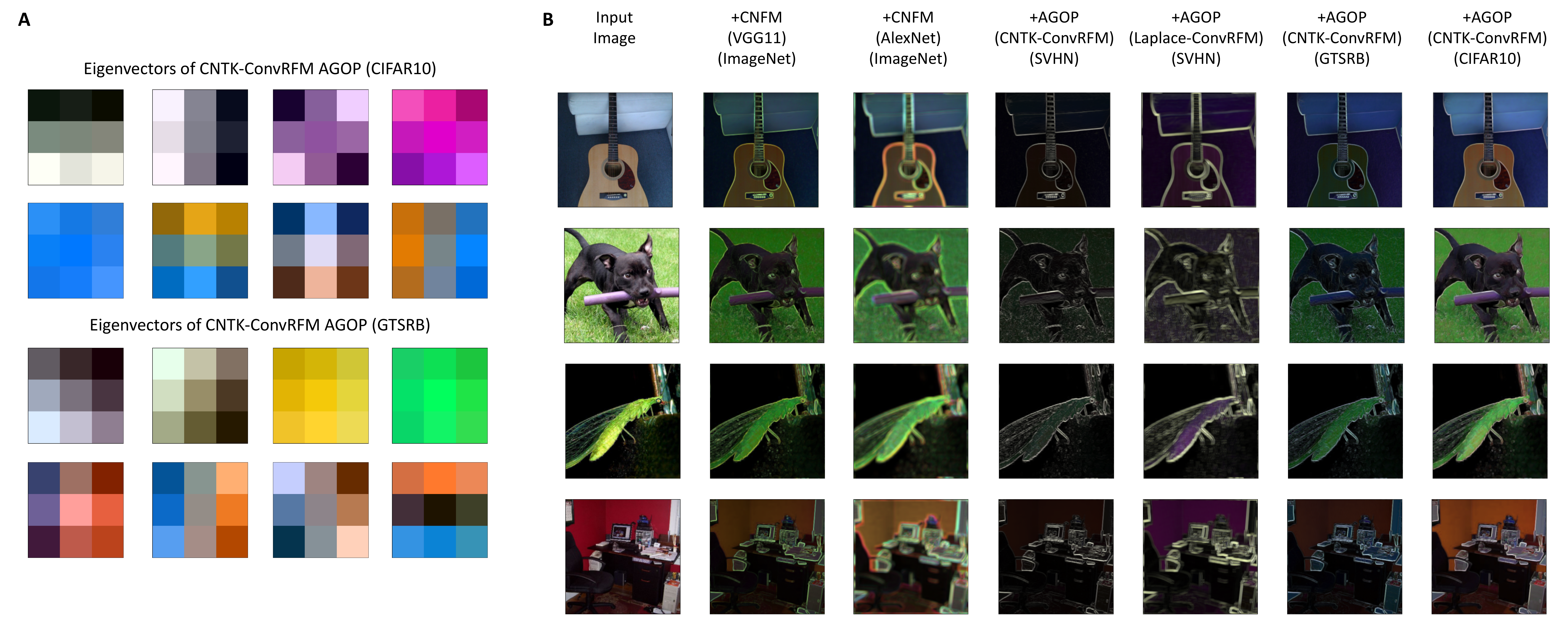}
    \caption{\textbf{A.} Visualization of top $8$ eigenvectors of patch-AGOP from CNTK-ConvRFM trained on CIFAR10 and GTSRB.  \textbf{B.} Comparison of feature extraction on ImageNet data performed by CNFMs from VGG11, AlexNet pre-trained on ImageNet and patch-AGOP operators of CNTK-ConvRFM trained on SVHN, GTSRB, CIFAR10 and Laplace-ConvRFM trained on SVHN.}
    \label{appendix fig: CIFAR10 GTSRB ConvRFM Features}
\end{figure}

\end{document}